\documentclass[letterpaper, 10 pt, conference]{ieeeconf}  

\IEEEoverridecommandlockouts                              

\usepackage{times}
\usepackage{soul}
\usepackage{url}
\usepackage[draft]{hyperref}
\usepackage[utf8]{inputenc}
\usepackage{subcaption}
\usepackage{graphicx}

\usepackage{amsmath,amssymb,amsthm,amsfonts}
\usepackage{mathtools}
\usepackage{booktabs}
\usepackage{algorithm}
\usepackage{algorithmic}
\usepackage{verbatim}
\usepackage[mathscr]{euscript}

\newtheorem{prop}{Proposition}
\DeclareMathOperator*{\Expect}{\mathbb{E}}

\usepackage{xspace}
\usepackage{color}
\usepackage{todonotes}

\title{\LARGE \bf
Deep R-Learning for Continual Area Sweeping
}

\author{
Rishi Shah\textsuperscript{*}, Yuqian Jiang\textsuperscript{*}, Justin Hart, Peter Stone \\
Department of Computer Science, The University of Texas at Austin, Austin, TX, USA\\
\{rishihahs, jiangyuqian\}@utexas.edu, \{hart, pstone\}@cs.utexas.edu
\thanks{* Equal contribution}
}

\begin{document}

\maketitle
\thispagestyle{empty}
\pagestyle{empty}

\begin{abstract}
Coverage path planning is a well-studied problem in robotics in which a robot must plan a path that passes through every point in a given area repeatedly, usually with a uniform frequency. To address the scenario in which some points need to be visited more frequently than others, this problem has been extended to non-uniform coverage planning. This paper considers the variant of non-uniform coverage in which the robot does not know the distribution of relevant events beforehand and must nevertheless learn to maximize the rate of detecting events of interest. This \emph{continual area sweeping} problem has been previously formalized  in a way that makes strong assumptions about the environment, and to date only a greedy approach has been proposed. We generalize the \emph{continual area sweeping} formulation to include fewer environmental constraints, and propose a novel approach based on reinforcement learning in a Semi-Markov Decision Process. This approach is evaluated in an abstract simulation and in a high fidelity Gazebo simulation. These evaluations show significant improvement upon the existing approach in general settings, which is especially relevant in the growing area of service robotics.
\end{abstract}
\section{Introduction}
Consider a service robot operating in an office or home. When a user requests that the robot bring a cold beverage or pick up the mail, the robot must reason about not only the static facts, such as the locations of rooms, but also the locations of objects in its environment which can change over time. As the occupants in this environment may move objects around, efficiently servicing these requests requires continually surveying the area to keep up to date on the objects’ locations.

The problem of \textit{continual area sweeping} was introduced by Ahmadi and Stone \cite{ahmadi2005continuous} as one motivated by building maintenance tasks in which some areas of the building see higher traffic and messier activities and therefore must receive more attention. A robot performing such tasks needs to service trash cans and restrooms more frequently than closets. For example, in a cleaning setting, the robot acts optimally when the time between the appearance of a mess and cleaning it up is minimal. Ahmadi and Stone \cite{ahmadi2005continuous} formalize a process of visiting areas of the map in a gridworld in which ``events'' (representing dirt and messes) appear non-uniformly throughout the robot's environment. Unlike more classical approaches, the distribution of these events is neither known nor constant, and thus must be learned online. Performance is measured based on how long it takes from the onset of an event to its servicing. 

To model the task of a service robot surveying its environment for changes, this paper extends continual area sweeping. In the original continual area sweeping formulation, the objective is to minimize the time to detect events. This paper additionally considers the objective of maximizing the number of events detected per second (DPS). Assumptions about the distribution and appearance of events are also relaxed in order to better represent this scenario. We introduce the \textsc{dps-max} approach to maximize detections per second. \textsc{dps-max} combines a novel formulation based on a Semi-Markov Decision Process in the average reward setting, and then a deep reinforcement learning algorithm to solve it.

Evaluations of our \textsc{dps-max} approach are presented in two simulation domains. An abstract gridworld is used to compare the performance of the Reinforcement Learning (RL) approach with the approach presented by Ahmadi and Stone
\cite{ahmadi2005continuous},
which serves as a baseline. Results show \textsc{dps-max} significantly improves performance in the most general scenario, and more flexibly handles complex event patterns by leveraging extra environmental information. \textsc{dps-max} is then evaluated on a simulated service robot in Gazebo on the task of detecting object placements, where we demonstrate that, unlike the baseline, \textsc{dps-max} is able to recognize previously seen geometric features between different environments.

The primary contribution of this paper is the \textsc{dps-max} approach which combines a novel Semi-Markov Decision Process problem formulation with a deep reinforcement learning algorithm to solve it. \textsc{dps-max} addresses a general class of continual area sweeping problems, specifically those motivated by the growing area of service robots. Under the assumptions reflective of such scenarios, \textsc{dps-max} significantly outperforms the prior state-of-the-art algorithm for continual area sweeping.

\section{Continual Area Sweeping}\label{sec:problem_formulation}
In the continual area sweeping task, a robot continually travels in an environment with the goal of detecting or reacting to events of interest. The environment is represented as a 2D map which is divided into a set of discrete grid cells $g \in G$. A set of events $e \in E$ can occur anywhere in the environment at any time $t$, and the distribution of events is unknown. The robot makes sequential decisions which are broken down into discrete decision steps $n \in \mathbb{N}$. At each decision step $n$, the robot can take an action $a_n$ to move to any reachable grid cell $g$ (including staying at the current cell). This action space focuses on the decision of where to visit, and assumes the shortest path will be taken. When an action is executed, the robot is able to detect any events in every grid cell along its path. The number of such detections is $d_n$. Note that the robot must physically travel from grid cell to grid cell, and as such may take a variable length of time to do so. Following this fact, the wall-clock time of decision step $n$ is denoted as $t_n$, and the problem is modeled as a Semi-MDP (see Section \ref{sec:smdp}).

\subsection{Metrics}
We define two metrics, \textit{average detection time (ADT)} and \textit{detections per second (DPS)}, each appropriate for a different class of applications.

\textit{Average detection time (ADT)} is the average time elapsed from occurrence to detection of events. More formally, let $o(e)$ denote the time when event $e$ occurs in the grid, and let $s(e)$ denote the time at which event $e$ was seen. If $e$ has never been detected, then let $s(e)$ be the current time. Then ADT is $\frac{1}{m}\sum_{i=1}^m (s(e_i) - o(e_i))$, where $m$ is the total number of events that have occurred. This metric is used in the original continual area sweeping formulation \cite{ahmadi2005continuous}. If the goal of the robot is to be highly responsive to emergencies, such as spilled drinks for a maintenance robot, then it is appropriate to minimize average detection time.

\textit{Detections per second (DPS)} is the average of the number of events detected per unit time, computed as $\frac{1}{t_n}\sum_{i=1}^n d_i$. If the goal of the robot is to maintain up-to-date information in its environment, then it should detect as many changes as possible over time. Thus, maximizing detections per second is more meaningful.

Both metrics are defined in the continual setting, so the most relevant observation is the long term average as $m$ and $n$ become arbitrarily large.

\subsection{Assumptions}

We assume that at each time step, the number of events in a grid cell $g \in G$ has an upper bound.
Events can also stop after they occur. For example, in the object tracking task, if a water bottle is placed on a desk, and its owner later picks it up, the event of the disappearance of the object overwrites the event of its appearance. In realistic domains, the bound on number of events in each grid cell is usually close to 1 for a fine enough grid representation.

\section{Related Work}
Coverage path planning is a family of problems in which an agent is given a map of its environment and must generate a navigational path that covers its environment. 
The family spans three main categories, which we survey below.

\subsection{Uniform Coverage}
Uniform coverage, also known as sweeping, has an agent generate a navigational path such that the agent passes through the entire volume of the map \cite{choset2001coverage,gabriely2001spanning,Galceran:2013:SCP:2542686.2542724}. This approach is useful for a variety of applications where the robot must travel over the entire area, such as lawn mowing or vacuum cleaning. 

\subsection{Adversarial Coverage}
Patrolling is a related problem that operates in a similar setting. Gatti \cite{gatti2008game} describes a game theoretic approach based on adversarial guard and robber agents that act strategically. Much of the other work in this area is concerned with such an adversarial two player game scenario \cite{basilico2012patrolling,bovsansky2011computing}. This work focuses on non-adversarial settings motivated by service robot environments where events from tasks such as cleaning or semantic mapping do not involve opponents. 

\subsection{Non-uniform Coverage}
Uniform coverage can be wasteful when areas of a map do not have equal importance. In such a setting, non-uniform coverage approaches optimize some metric by giving more focus to certain areas. 

\subsubsection{Coverage with Metrics}

Ergodic coverage aims to optimize the ergodicity metric, in which time spent in a given area is related to the spatial distribution of regions of interest over that area \cite{ayvali2017ergodic}. Information surfing is also a non-uniform coverage approach but instead seeks to maximize discriminatory information by planning a path that exploits local information gradients~\cite{infosurfing}. 
Most non-uniform coverage approaches focus on planning with a given distribution of events. 

In some applications such as surveillance, a robot may need to cover an area repeatedly with an unknown or changing distribution of events. More recent work relaxes the assumption of having a known events distribution in ergodic coverage. Mavrommati et. al. present an adaptive planning approach that works with a changing events distribution while still optimizing the ergodicity metric \cite{mavrommati2017real}. Continual area sweeping is a different non-uniform coverage problem where an area has to be covered repeatedly with an unknown and changing event distribution while optimizing the ADT or DPS metric. Unlike ergodic coverage which only cares about how much time is spent proportionally in each region, in continual area sweeping \emph{when} each region is visited is also important.


\subsubsection{\textsc{adt-greedy} Algorithm}\label{adtgreedydefinition}
The closest work to this paper is by Ahmadi and Stone \cite{ahmadi2005continuous} who introduced the non-uniform continual area sweeping problem and proposed a greedy algorithm that minimizes the average detection time (ADT) while learning a changing distribution of events. For the remainder of this paper, this approach will be referred to as \textsc{adt-greedy}. 

\textsc{adt-greedy} makes two assumptions that are revisited in the current work. The first is an assumption that at each time step there is a fixed probability $p_g$ for an event occurring at grid cell $g$. It follows that the number of events in each grid cell follows a binomial distribution $B(t, p_g)$, where $t$ is the number of time steps since the cell was last visited. Henceforth, this will be referred to as the \emph{binomial assumption}. The second assumption is that there is no upper bound on the number of events per cell, henceforth called the \emph{unbounded events assumption}.
A convenient consequence of these assumptions is that the expected number of events in a cell is linear in the time since the cell was last visited. This linearity exists because the expectation with respect to the Binomial distribution $B(t, p_g)$ is $t \cdot p_g$. Ahmadi and Stone \cite{ahmadi2005continuous} show that under these conditions maximizing the total expected number of detected events is the same as minimizing ADT.
The \textsc{adt-greedy} algorithm consists of a learning module that learns $p_g$ for every grid cell $g$. A planning module then greedily chooses the target cell that leads to the path with the highest expected number of events. 

This paper shows that \textsc{adt-greedy} leads to suboptimal behavior when the assumptions are violated, which motivates using reinforcement learning to maximize continual area sweeping metrics without directly learning the event distribution.

\section{Approach}\label{sec:approach}

This section proposes a novel formulation of continual area sweeping as a Semi-Markov Decision Process which is then solved using a deep RL approach. The combination of this novel formulation and algorithm to solve it is called \textsc{dps-max}, which provably maximizes average detections per second (DPS).

\subsection{Semi-MDP Model}\label{sec:smdp}

The proposed \textsc{dps-max} approach is the combination of a novel formulation of the continual area sweeping problem as a Semi-Markov Decision Process (SMDP) along with a deep reinforcement learning algorithm to solve it. This section describes the SMDP formulation, which consists of $(\mathcal{S}, \mathcal{A}, R, P)$:

    $\mathcal{S}$ is the state space. Each state consists of three components:
    
        \par \textbf{2D costmap}: Notated as $G$ in Section \ref{sec:problem_formulation}, the discretized grid of the environment is included. This grid is represented by a 2D array where a cell that has an obstacle is given a value of $1$, and others are given a $0$.
        \par \textbf{Robot position}: The robot's position in the environment is represented with a grid where the robot's current cell is $1$ and the remaining cells are $0$.
        \par \textbf{Event uncertainty}: Suppose $t$ seconds have passed since the robot has visited a particular cell. When $t = 0$, it is known that the robot has seen all of the events in the cell, but as $t$ increases so does uncertainty. Encoding this uncertainty allows the robot to take events it has seen into consideration when making decisions. Under a Poisson distribution, the probability that $0$ new events have appeared in a cell after time $t_d$ is $\exp{(-\alpha t_d)}$, with $\alpha$ the rate at which events appear. Each grid cell is filled in with this probability. $t_d = \infty$ if the cell has never been seen. This formulation does not assume that event appearance is exactly Poisson; rather, these probabilities provide initial information that the function approximator can later use to learn the true dynamics of event appearance.

    Combined with a CNN function approximator, these grid representations encode the assumption that local spatial regions of the state should be similar (see Section \ref{sec:nnet}).
    
    $\mathcal{A}$ is the action space, which includes all empty cells in $G$, or in other words, the free spaces that the robot can navigate to. The robot takes one action $a_n$ at each decision step $n$. The key difference between an SMDP and an MDP is that different actions can have different durations. It takes time for the robot to physically move, so actions that move the robot to a far away cell take more time than actions that cause the robot to navigate a shorter distance.
    
    $P: \mathcal{S}\times\mathcal{A} \to \mathcal{S}$ is the transition kernel. A key component of $P$ is the unknown probability of event appearance.
    
    $R: \mathcal{S}\times\mathcal{A}\times\mathcal{S} \to \mathbb{R}$ is a measurable function denoting the reward given for a transition, defined in section \ref{sec:rewardconstruction}.

A stationary policy $\pi$ describes the action to take in a given state, and is thus a map from $\mathcal{S}$ to probability measures on $\mathcal{A}$.

The most common way to deal with non-episodic tasks is discounting future rewards. In the continual setting, the goal is not to maximize total rewards, but rather to optimize long-term averages. Thus we instead use the following average reward formulation:

$\rho^\pi$ is the average reward function:

\begin{equation}\label{averagereward}
    \rho^\pi(\mu) \coloneqq \liminf_{n \to \infty} \frac{1}{n} \Expect{\Big [ \sum_{k = 0}^{n-1} R(s_k, a_k, s_{k+1}) \Big ]}
\end{equation}

\noindent where $\mu$ is an initial state distribution, and the expectation is taken with respect to the appropriate measure derived from $\pi$ and $\mu$ \cite{feinberg1996measurability}. For convenience, when $\mu(s) = 1$ for some state $s$, we use the notation $\rho^\pi(s)$.

The optimal differential value function, then, is:
\begin{equation*}
    Q^*(s, a) = \mathbb{E}_P{R(s, a, s') - \sup_{\pi}\rho^{\pi}(s) + \mathbb{E}_P{\big [ \max_{a' \in \mathcal{A}} Q^*(s', a') \big ]}}
\end{equation*}

The goal under this formulation is to approximate $Q^*$, from which we can derive an optimal policy.

\subsection{Reward Construction} \label{sec:rewardconstruction}
The reward construction of \textsc{dps-max} should maximize average detections per second (DPS). Note that it is not sufficient to construct a reward function where the value is 1 if an event was detected, 0 otherwise. Average reward is maximized as in Equation \ref{averagereward}, which maximizes the average number of detections per decision step. Since the problem formulation is as an SMDP, however, maximizing \textit{detections per decision step} is not the same as maximizing \textit{detections per second}. Special care is needed, because actions take different lengths of time, making these two metrics not even approximately similar.

Reward construction is an important part of SMDP design, and many schemes deal with handling the time and decision step mismatch \cite{baykal2010semi}. The new reward function is designed specifically for the case of optimizing a rate, such as detections per second.

\begin{prop}
Take $\{(s_n, a_n)\}_{n \geq 0} \subset \mathcal{S}\times\mathcal{A}$ to be a trajectory generated from a policy $\pi$. Let $\{\phi_n\}_{n \geq 0} \subset \mathbb{R}$ be a sequence, and $\{t_n\}_{n \geq 0} \subset \mathbb{R}$ be an increasing sequence denoting the associated environmental time. Construct $R$ in the following way:

\begin{align*}
& R(s_0, a_0, s_1) \coloneqq 0 \\
& R(s_n, a_n, s_{n+1}) \coloneqq (n+1)\frac{\phi_{n+1}}{t_{n+1}} - n\frac{\phi_n}{t_{n}}
\end{align*}

Then $\rho^\pi(s_0) = \liminf_{n \to \infty} \frac{\Expect{\phi_n}}{t_n}$
\end{prop}

\begin{proof}
$\\$
Substituting in (\ref{averagereward}): 
\begin{align*}
    \rho^\pi(s_0) &= \liminf_{n \to \infty} \frac{1}{n} \Expect{\Big [ \sum_{k = 0}^{n-1} (k+1)\frac{\phi_{k+1}}{t_{k+1}} - k\frac{\phi_k}{t_k} \Big ] }
\end{align*}

The sum telescopes, leading to:
\begin{align*}
    \rho^\pi(s_0) &= \liminf_{n \to \infty} \frac{1}{n} n\frac{\Expect{\phi_n}}{t_n} 
\end{align*}
\end{proof}

The corollary of this proposition is that setting $\phi_n$ to be the number of detections seen at decision step $n$ provably optimizes average detections per second.

\subsection{Deep R-Learning}
R-Learning is a classical approach for learning an optimal differential value function \cite{schwartz1993reinforcement}. Its purpose is to handle infinite-horizon tasks where finding a policy to maximize average reward is more meaningful than temporal discounting. For this problem, discounting is not a good fit as in order to optimize average detections per second (DPS), detections in the future cannot be considered less valuable. To represent the value function, a suitable function approximator is needed. We introduce a deep neural network variant of R-Learning based on double DQN \cite{van2016deep}, which allows for the integration of neural networks with double Q-Learning.

Algorithm \ref{alg:deeprl} describes the new algorithm. The key changes to double DQN are highlighted here. First, the target in line 9 reflects the R-Learning update by subtracting out the running average reward estimate. Lines 11 and 12 compute the change to $\rho$. Here, the TD errors of the batch are averaged so long as the actions taken were close to optimal. As a result $\delta$ essentially controls a bias-variance trade off of average reward updates. A low $\delta$ will lead to lower bias as it is closer to approximating $\rho^{\pi^*}$, but there will be higher variance as it takes smaller batch averages. If line 12 attempts to take the average of an empty set, then the subsequent if-statement will not execute.

\begin{algorithm}[tb]
\caption{Deep R-Learning}
\label{alg:deeprl}
\begin{algorithmic}[1] 
\STATE Initialize empty experience replay buffer $\mathcal{D}$.
\STATE Initialize network $Q$ with random weights $\boldsymbol{\theta} = \boldsymbol{\theta^-}$.
\STATE Initialize $\rho = 0$.
\FOR{t = 1, \dots, M}
\STATE Select an action $a_t$ according to an action selection mechanism like $\epsilon$-greedy.
\STATE Execute $a_t$ and store the resulting transition $(s_t, a_t, r_t, s_{t+1})$ in $\mathcal{D}$.
\STATE Randomly sample a batch of transitions $\{(s_j, a_j, r_j, s_{j+1})\}$ from $\mathcal{D}$.
\STATE Let $q_{max} = Q \left( s _ { j + 1 } , \operatorname { argmax }_a Q \left( s_{ j + 1 } , a ; \boldsymbol { \theta } \right) ; \boldsymbol { \theta^- } \right)$
\STATE Let $y_j = r_j - \rho + q_{max}$
\STATE Take a gradient descent step on $\big (y_j - Q \left( s _j , a_j ; \boldsymbol { \theta } \right) \big )^2$.
\STATE Let $\Delta_j = y_j - Q \left( s _j , a_j ; \boldsymbol { \theta } \right)$
\STATE Let $\Delta = \operatorname{avg}\{\Delta_j \  \text{s.t.} \  |Q(s_j, a_j)-\max_a Q(s_j, a)| < \delta\}$
\IF {$\Delta$ is well-defined}
\STATE $\rho = \rho + \alpha\Delta$ for learning rate $\alpha$
\ENDIF
\STATE Every $\tau$ steps, set $\boldsymbol{\theta^-} = \boldsymbol{\theta}$.
\ENDFOR
\end{algorithmic}
\end{algorithm}

\subsection{$Q$ Function Representation}\label{sec:nnet}
To represent $Q$, Algorithm \ref{alg:deeprl} uses an encoder-decoder network as a way of exploiting the topology of $\mathcal{S}$ and $\mathcal{A}$. For a practical map, there can be millions of actions, since the agent can choose to move anywhere (resulting in close to height$\times$width of $G$ number of actions). Value based methods are normally poorly suited for such a large action space, but this choice of architecture overcomes that limitation. Due to convolutional layers, updates made to the $Q$-value of a state-action pair immediately generalize to a local neighborhood.
Figure \ref{fig:network} illustrates the architecture, which shows an encoder-decoder network where the environment map, robot position, and event uncertainty are represented as grids and fed in as the input. The output is the action-value for each cell (action) in the map. The max-pool and upsampling layers help coalesce the action values of neighbors.

\begin{figure}[h]
    \centering
    \includegraphics[width=0.75\columnwidth]{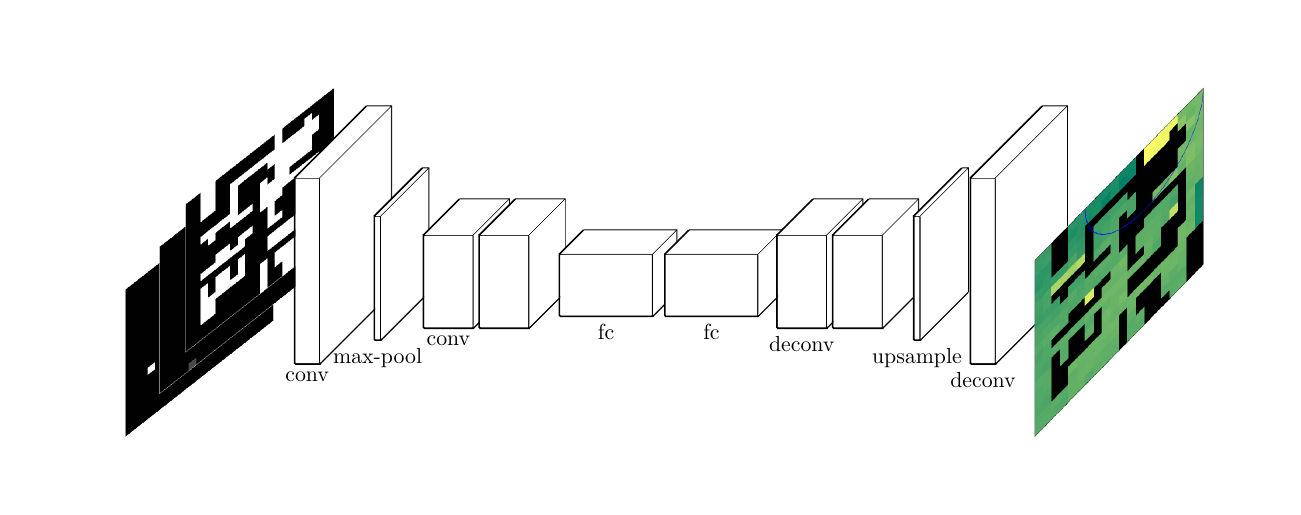}
    \caption{
    In these experiments, the architecture comprises a $32{\times}5{\times}5@3$ conv, $2{\times}2@2$ max-pool, two $16{\times}4{\times}4@2$ conv, and a $500$ unit fully connected layer, while the decoder part is the reverse with the conv layers swapped for deconv layers and upsampling for the max-pool.}
    \label{fig:network}
\end{figure}

\section{Experiments}
\label{sec:experiments}
We evaluate \textsc{dps-max} on two domains: an abstract gridworld, and a simulated house in the Gazebo robotics simulator. The gridworld is used to compare the performances of the \textsc{dps-max} approach and the \textsc{adt-greedy} algorithm under different environmental assumptions. We hypothesize that (1) \textsc{dps-max} will outperform \textsc{adt-greedy} when the environment violates the binomial assumption and the unbounded events assumption as defined in section \ref{adtgreedydefinition}, (2) \textsc{adt-greedy} will do better under the ADT metric, while \textsc{dps-max} will do better under the DPS metric, and (3) \textsc{dps-max} is able to incorporate knowledge that influences event appearance into its state space, thereby leading to improved performance.

\begin{figure}[t]
     \centering
    \begin{subfigure}{0.45\columnwidth}
        \centering
        \includegraphics[width=.9\columnwidth]{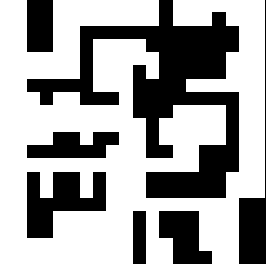}
        \caption{}
        \label{fig:gridworld}
    \end{subfigure}
    \begin{subfigure}{0.45\columnwidth}
        \centering
        \includegraphics[width=.9\columnwidth]{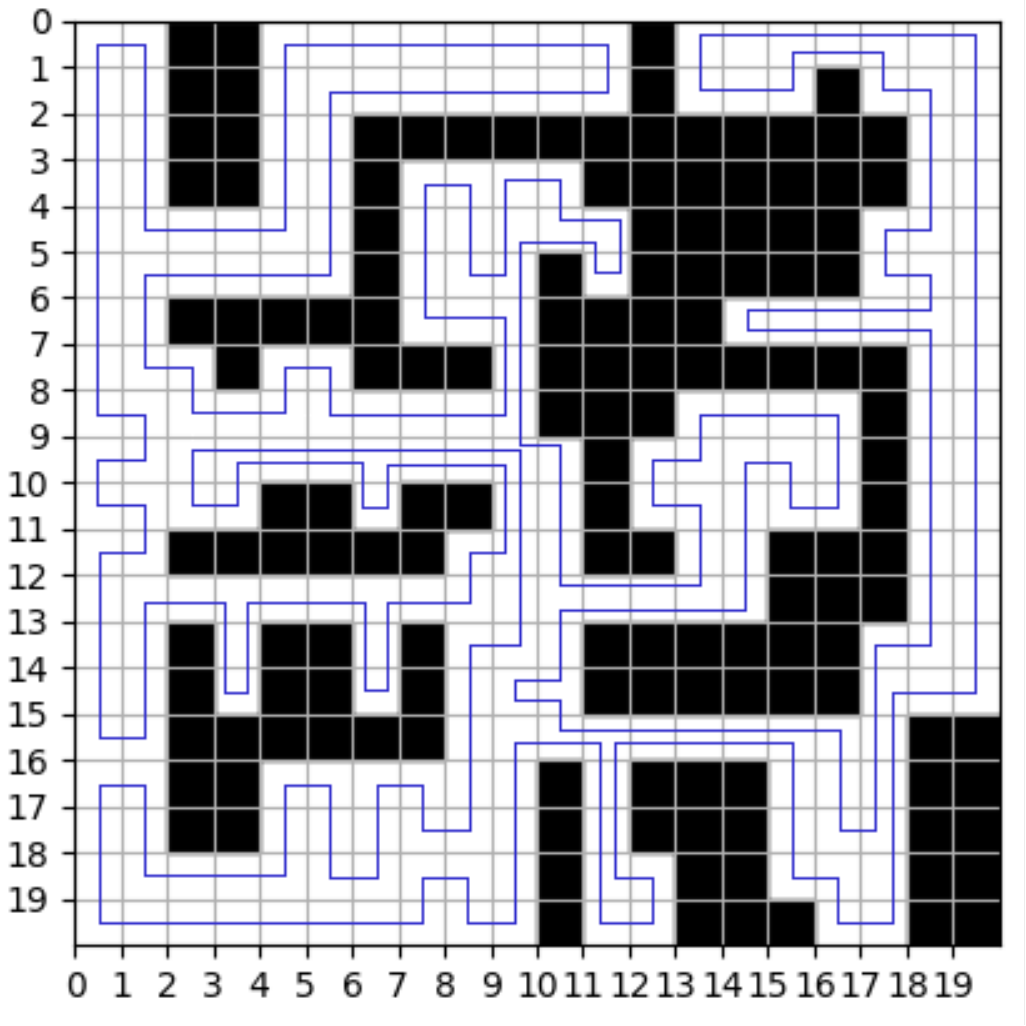}
        \caption{}
        \label{fig:path}
    \end{subfigure}
    \caption{
    (a) Gridworld where black represents walls. In 5 random cells, events have a periodicity ranging from 10 to 50 seconds in the periodic case, or show up with a fixed probability ranging from $1/10$ to $1/50$ per second in the binomial case. (b) One coverage patrolling path for (a).}
\end{figure}

\subsection{Gridworld Experiments}

This evaluation tests on a gridworld in order to evaluate the performance of the proposed algorithm under the different environmental assumptions and the two metrics. We test our hypothesis that \textsc{dps-max} outperforms \textsc{adt-greedy} when assumptions of \textsc{adt-greedy} are violated, and study how much the ADT and DPS metrics align with each other.

\subsubsection{Setup}

Figure \ref{fig:gridworld} illustrates the setup for the following gridworld experiments.
A $20 \times 20$ grid is populated with random locations at which events may occur. In each cell, either events occur periodically, or the number of events follow the binomial distribution. In the binomial case, events appear with a fixed probability between $1/10 - 1/50$ each time step, and in the periodic case, according to a fixed period between $10 - 50$ time steps. These events occur in $1$ of $5$ fixed locations which are randomly generated at the start of each experiment, with a probability or time period associated with each of the $5$ locations at the start of the experiment. We also evaluate the effects of the bound on the number of events per grid cell by varying the bound from 6 to 1. This set of experiments tests the effect of the unbounded assumption made by the \textsc{adt-greedy} algorithm, with 6 being closer to the original assumption and 1 completely violating it.

For each configuration, 8 grids of random event positions and occurrence probabilities/periods are generated. Since the instances have randomly generated event positions, the best achievable DPS and ADT are different for each instance. We compare \textsc{dps-max} to \textsc{adt-greedy} by taking the percentage difference in our DPS or ADT over the DPS or ADT of \textsc{adt-greedy}, averaged across each configuration.

In this set of experiments, the learning rate $\alpha$ in Algorithm~\ref{alg:deeprl} is set to 0.0001. The exploration strategy is to initialize the agent in a random position, run $\epsilon$-greedy exploration for 50 steps, and then reset the agent to a random position. The low learning rate and the frequent resets are used to ensure sufficient exploration.\footnote{Otherwise, exploration tends to stick around grid cells with frequent events, and not cover enough of the state space; causing high variances in evaluations since the robot's initial position is random.} The stopping criterion for training is the following: after every 20,000 training steps, the model is evaluated by executing the policy at a random initial position, and training terminates if the DPS has not improved in the last 10 evaluations.\footnote{On average, training terminated in around 400,000 steps.}

\subsubsection{Results}
Figure~\ref{fig:adt} shows the average percentage difference in DPS, where higher than $0$ means \textsc{dps-max} detects more events per unit time than \textsc{adt-greedy}. Figure~\ref{fig:dps} shows the average percentage difference in ADT, where lower than $0$ means on average \textsc{dps-max} takes less time between event appearance and detection than \textsc{adt-greedy}. The error bars in the figures report standard deviation.

\begin{figure}[t]
    \centering
    \begin{subfigure}{0.75\columnwidth}
        \centering
        \includegraphics[width=.9\columnwidth]{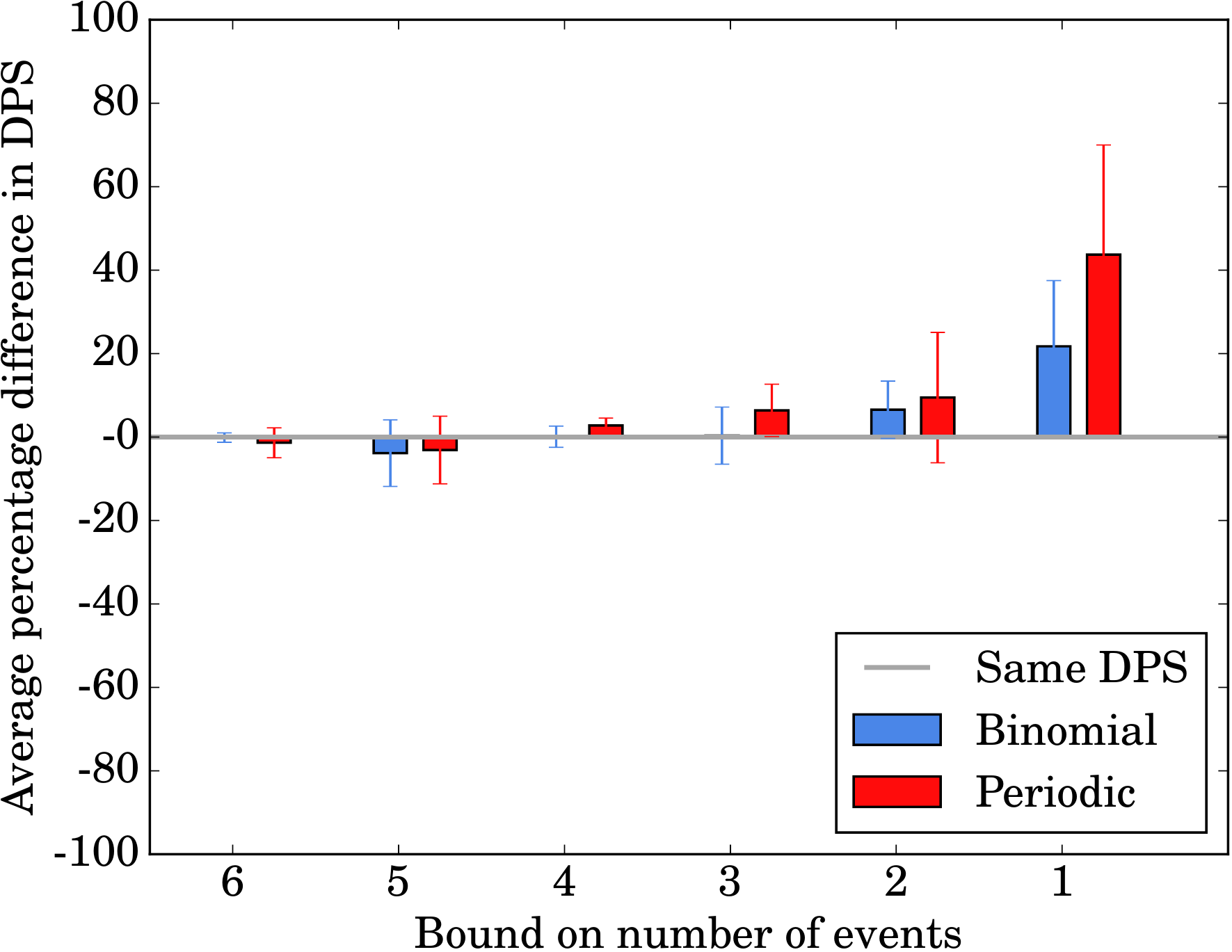}
        \caption{}
        \label{fig:adt}
    \end{subfigure}
    
    \begin{subfigure}{0.75\columnwidth}
        \centering
        \includegraphics[width=.9\columnwidth]{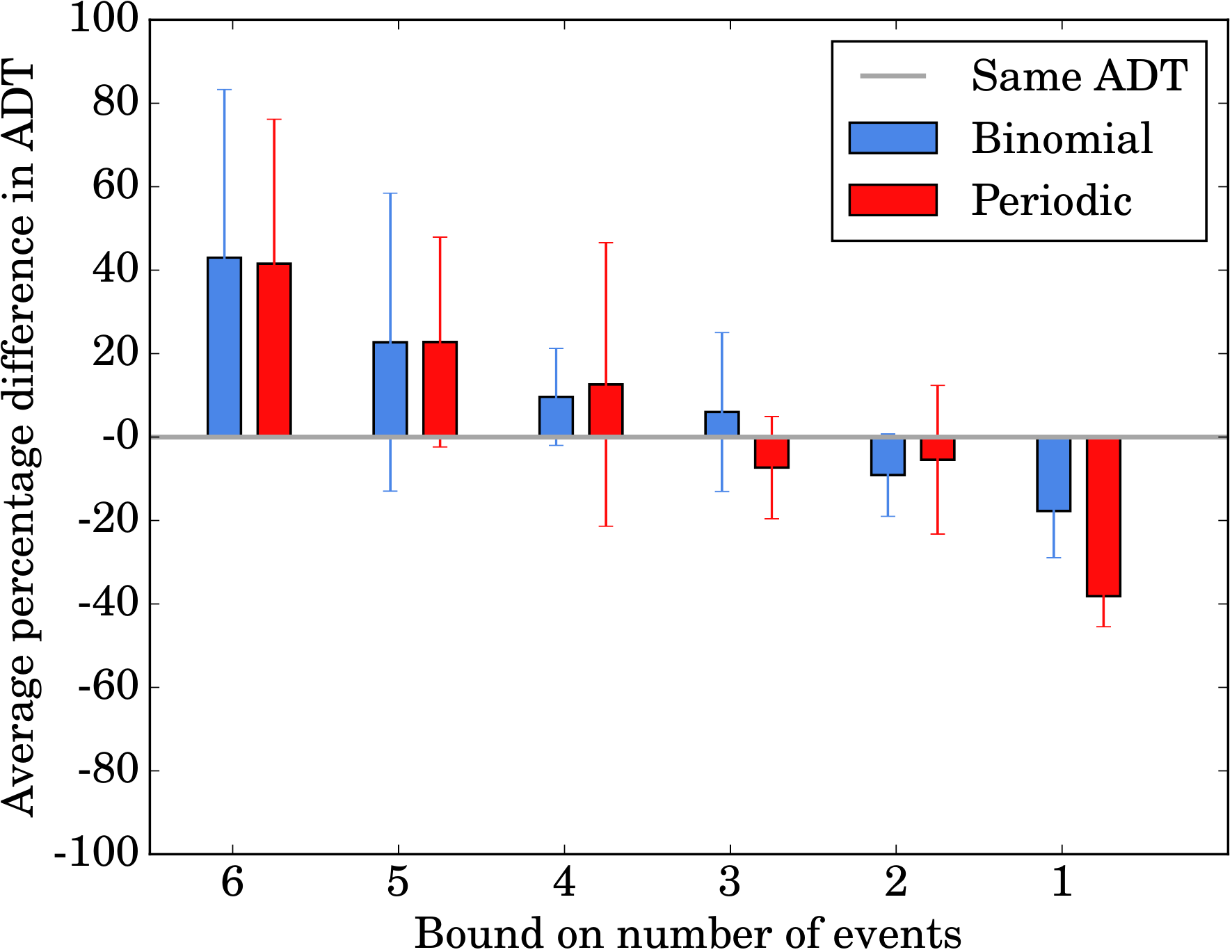}
        \caption{}
        \label{fig:dps}
    \end{subfigure}
    \caption{(a) Average percentage difference in detections per second (DPS) of \textsc{dps-max} over \textsc{adt-greedy}, where higher than $0$ means \textsc{dps-max} detects more events per unit time than \textsc{adt-greedy}. (b) Average percentage difference in average detection time (ADT) of \textsc{dps-max} over \textsc{adt-greedy}, where lower than $0$ means \textsc{dps-max} takes less time between event appearance and detection than \textsc{adt-greedy}.}
\end{figure}

As shown by both figures, \textsc{dps-max} has the most advantage over the \textsc{adt-greedy} approach when the binomial and unbounded events assumptions are most violated. When event appearance is periodic and the number of events in each cell is bounded by $1$, \textsc{dps-max} achieves the best improvement in DPS ($43.7\%$) and the most reduction in ADT ($38.4\%$). The reduction in ADT is surprising because unlike \textsc{adt-greedy}, \textsc{dps-max} does not directly optimize for ADT. In fact, the two metrics align except for a few cases. For instance, when the bound is $4$ in the periodic setting, \textsc{dps-max} has better DPS but worse ADT compared to \textsc{adt-greedy}.

When the bound on events in every grid cell is high, \textsc{dps-max} does not outperform \textsc{adt-greedy} on either metric. One possible explanation is that detecting many events in one step gives a large reward, which causes instability in learning. Such a scenario is not the focus of this work, so we leave further investigation of this case to future work. 

For comparison, we tested coverage patrolling \cite{choset2001coverage}, which ignores the dynamics of event appearance by covering all grid cells. In the case of binomial event appearance and bound = 1, following the path in Figure~\ref{fig:path} leads to 266.7\% in average percentage difference in ADT and -65.8\% in average percentage difference in DPS compared to \textsc{adt-greedy}. As expected, this performance is worse on both metrics than \textsc{adt-greedy} and \textsc{dps-max}, showcasing the advantage of efficiently biasing travel time in favor of cells with frequent events.

\subsection{Incorporating Extra Knowledge}
Leveraging information about external factors in the environment can improve performance. Consider a person moving around in a building and doing activities like throwing away trash or moving belongings around. In this setting, the robot would benefit from incorporating knowledge of where the person is, as the appearance of events is highly correlated. While \textsc{adt-greedy} does not provide a way to add such knowledge, \textsc{dps-max} does by adding information to the state. To illustrate, we conduct an experiment in which a person walks randomly on the grid in Figure~\ref{fig:gridworld} and causes an event with a $30\%$ chance in each step, with the number of events in each cell bounded by 1. The algorithm is the same as above with the only difference being that a grid with the location of the person is added to the state for \textsc{dps-max}.
With the location of the person incorporated, \textsc{dps-max} achieves a 3922.6\% increase in DPS over \textsc{adt-greedy} and a 39.2\% reduction in ADT. Without the extra information, the average increase in DPS is 631.7\% and \textsc{dps-max} has worse ADT than \textsc{adt-greedy} with an average increase of 1029.2\%.
Thus, the formulation of \textsc{dps-max} allows the function approximator to learn the association between the person and events leading to significantly better performance.

\subsection{Recognizing Geometric Features}
Another desirable feature for a continual area sweeping agent is the ability to recognize previously seen geometric features. If objects are regularly placed on some piece of furniture, and the robot sees the furniture move during training, then the robot should naturally recognize that events that used to occur near the old location of the furniture are now likely to occur near the new location. This recognition can be accomplished by \textsc{dps-max} due to its convolutional neural network's ability to extract local features from a grid-based state. In contrast, \textsc{adt-greedy} by design forgets about learned events which is necessary as there is neither a feature-based representation of the map nor a feature extractor, so new information about grid cells must override old information.

\subsubsection{Simulation Environment}
Gazebo is a high fidelity robot simulator \cite{Koenig04designand} that we use to simulate the Toyota Human Support robot in an indoor environment. This set of experiments presents a realistic simulation of a robot. First, actions taken in Gazebo are noisy; the same action can take varying amounts of time to execute, and actions sometimes fail, causing the robot to stop midway through. Additionally, the environment map is large (300x300 grid representing a 900 square meter area). Successful learning in Gazebo requires capabilities similar to those of a real robot. 


\subsubsection{Setup}
The robot is placed in an empty room with a cubicle. The robot is trained off-policy where data is gathered by having the robot navigate through a human generated path roughly covering the whole room. During this training period, the robot sees the cubicle from different positions. As these experiments primarily concern the navigational aspects of the problem, robot perception is bypassed and, instead, events are classified as having been detected if the robot passes within 2 meters. The following results visualize the learned policy adapting to changes in the cubicle's location without further learning.


\subsubsection{Events dependent on furniture}
In the first test, an event constantly appears in a cubicle. This cubicle is moved, and the robot experiences this change in position during training. There is also a second event continually firing at a fixed position. After training, the learned policy is run with the cubicle in two positions that the robot had seen earlier during training. The path from executing the policy with no exploration is shown in the top row of Figure \ref{fig:furndeep}, which represents the path that would result if the robot always moved to the location with the greatest action-value, but in practice the robot would cover the whole area using an exploration strategy such as $\epsilon$-greedy or softmax action selection. The fact that the \emph{same} policy produced tailored paths depending on the location of the cubicle shows that the robot was able to associate the visual appearance of the cubicle with the fact that events often appear there.

\begin{figure}[t]
    \centering
    \includegraphics[width=.35\columnwidth]{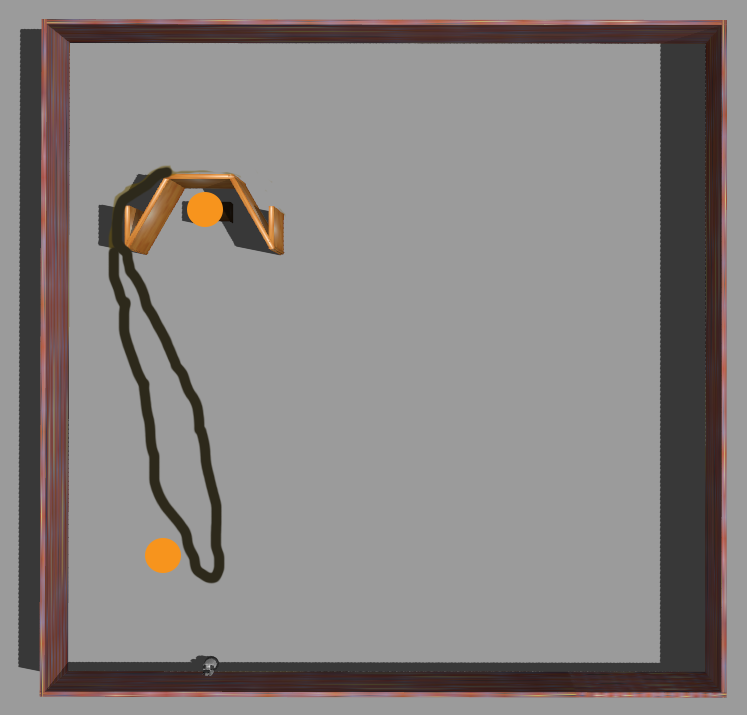}
    \includegraphics[width=.35\columnwidth]{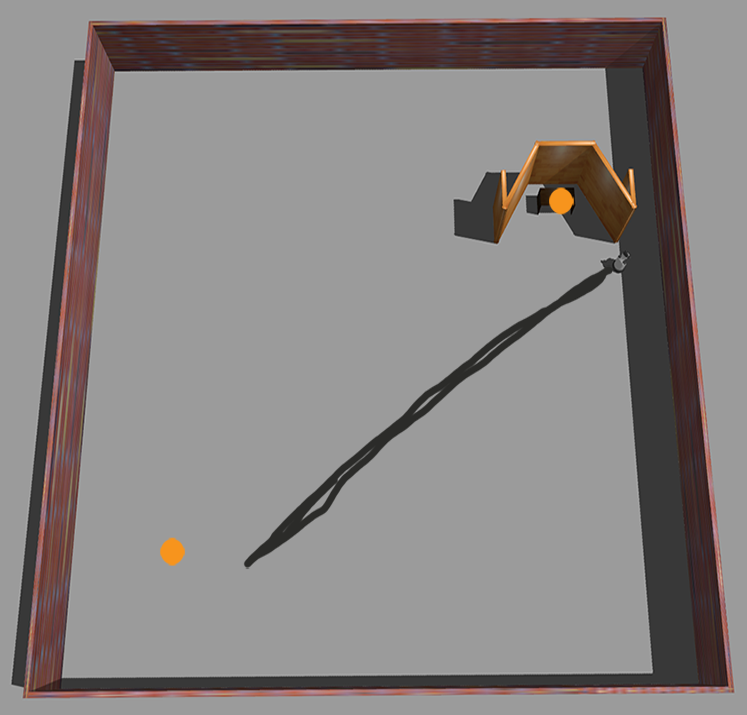}
    
    \includegraphics[width=.35\columnwidth]{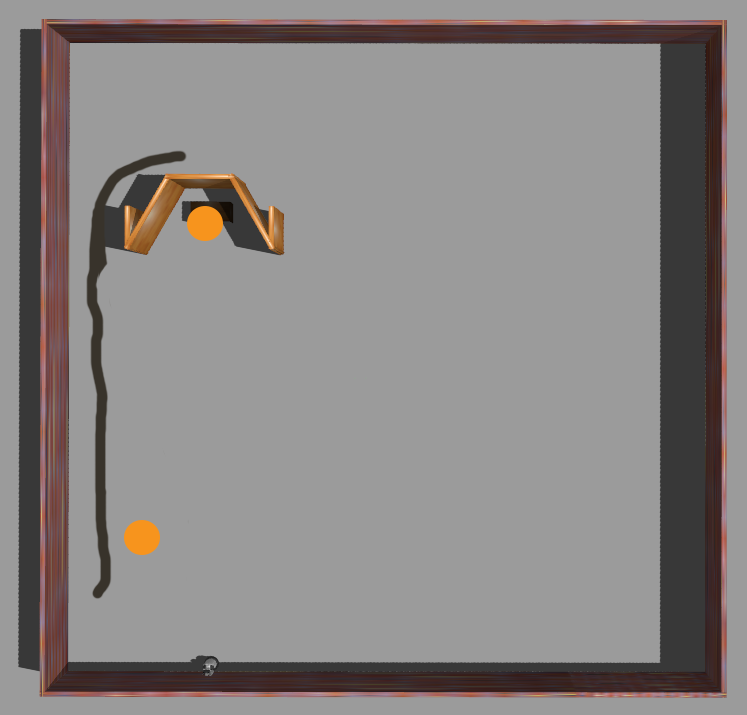}
    \includegraphics[width=.35\columnwidth]{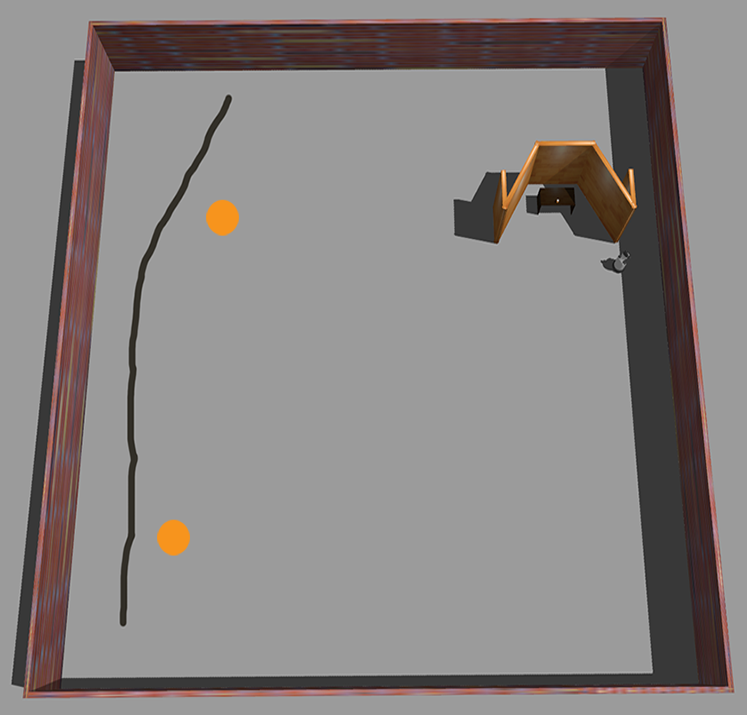}
    \caption{\textbf{Top:} Events dependent on furniture. \textbf{Bottom:} Events independent of furniture. Black line shows the path from executing the learned policy without exploration. Orange dots represent where events occur.}
    \label{fig:furndeep}
\end{figure}

\subsubsection{Events independent of furniture}
Remembering geometric features is not useful if the robot constructs false associations. If the cubicle moves, but events do not move with the cubicle, then the robot should simply ignore the geometry of the cubicle. To test this ability, the previous experiment is repeated, but with both objects fixed. The same learned policy has the robot ignore the moving cubicles as seen in the bottom row of Figure \ref{fig:furndeep}.

These two experiments show that the robot is able to memorize geometric features and recall them on demand, but only when they are truly relevant.

\section{Conclusion}

This paper extends the formulation of the continual area sweeping problem using an SMDP, and proposes a deep R-learning approach to maximize average detections per second. These two components comprise the main contribution of this paper, which is the introduction of the novel \textsc{dps-max} approach for the general class of continual area sweeping problems that we expect to arise frequently in the growing area of service robotics, and a demonstration that, under the assumptions most reflective of such scenarios, \textsc{dps-max} significantly outperforms the prior state-of-the-art algorithm for continual area sweeping. 
Furthermore, it is shown that \textsc{dps-max} can discover structure in event occurrence (such as geometric features) and leverage extra state information (such as the location of a person).
In future work, we plan to apply and test this approach on real service robots as a background default behavior that improves knowledge of the environment when no other task is active.

\addtolength{\textheight}{-12cm}   


\bibliographystyle{IEEEtran}
\bibliography{IEEEabrv,references}

\end{document}